\documentclass{article}
\usepackage{spconf,amsmath,amssymb,graphicx,bm,dsfont}
\graphicspath{{./}{../}}
\usepackage{xcolor}
\definecolor{hcpink}{RGB}{225,40,150}
\usepackage{booktabs}
\usepackage{mysymbol}
\usepackage{algorithm}
\usepackage{algorithmic}
\usepackage[hidelinks]{hyperref}
\usepackage{mathtools} 

\newcommand{\EE}{\mathbb{E}}

\newcommand{\clip}{\mathrm{clip}}
\newcommand{\pos}[1]{\left[#1\right]_{+}}

\title{Feasible Flow Matching for Graph Reconstruction via Within-Sampling Primal-Dual Guidance}

\name{Haoming Chen$^{\dagger}$ \qquad Nicolas Zilberstein$^{\dagger}$ \qquad Santiago Paternain$^{\star}$ \qquad Santiago Segarra$^{\dagger}$\thanks{Emails: \{hc80, nzilberstein, segarra\}@rice.edu, paters@rpi.edu}}
\address{$^{\dagger}$Rice University, Houston, TX, USA \qquad $^{\star}$Rensselaer Polytechnic Institute, Troy, NY, USA}

\begin{document}
\ninept
\maketitle

\begin{abstract}
Graph reconstruction from partial observations often comes with structural side information, such as degree bounds, triangle counts, or an edge-density band.
Prior-Informed Flow Matching (PIFM) reconstructs graphs by transporting a local prior toward the graph distribution, but it provides no mechanism to incorporate this side information.
We put forth Constrained Primal-Dual PIFM (CPD-PIFM), which augments the sampler with Lagrange multipliers that evolve along each trajectory.
The multipliers respond to constraint violations at a predicted endpoint and guide subsequent sampling steps without retraining.
We prove that the sampler inherits PIFM's permutation equivariance and bound its expected terminal slack by a term that decays as the inverse square root of the number of steps, plus two approximation terms.
On three link-prediction benchmarks and nine combinations of datasets and constraints, CPD-PIFM raises feasibility by 11--26 percentage points and remains competitive with fixed guidance without selecting a separate multiplier for each constraint.
\end{abstract}


\section{Introduction}
\label{sec:intro}

Reconstructing the topology of a graph from partial observations is a longstanding inverse problem in signal processing and machine learning~\cite{segarra2017network, dong2016learning, mateos2019connecting}.
Beyond the observed edges, a practitioner may know that the graph respects degree constraints, contains a minimum number of triangles, or lies in an edge-density band.
A reconstruction that violates such knowledge is structurally incorrect even if its average edge-wise accuracy is high.
When the constraints describe properties of the true graph, enforcing them incorporates additional information into an ill-posed inverse problem~\cite{roddenberry2021network, navarro2022joint}.

Guidance mechanisms can incorporate structural requirements into trained generative models.
For graphs, Graph Guided Diffusion (GGDiff)~\cite{tenorio2025graph} uses a reward whose weight is tuned, while PRODIGY~\cite{sharma2024diffuse} projects iterates onto a feasible set.
The latter requires a tractable projection, which may be unavailable for the desired structural constraints.
Adaptive dual variables have also been used during sampling for inverse problems~\cite{chung2023solving, kim2025dual}, discrete decoding~\cite{tomasi2026primaldual}, and continuous-diffusion generation~\cite{hadou2026constrained}.
These methods motivate adapting constraint pressure within a sampling trajectory.

For reconstruction from partial graph observations, Prior-Informed Flow Matching (PIFM)~\cite{chen2026pifm} combines a local structural prior with a rectified flow~\cite{liuflow, lipmanflow}.
Motivated by the distortion--perception tradeoff~\cite{blau2018perception, freirich2021distortion}, it transports the distribution of prior estimates toward the graph distribution.
However, PIFM does not explicitly enforce additional structural constraints.
Adding fixed penalties requires choosing weights, sampling, checking feasibility, and retuning; and searching for separate weights becomes costly as the number of constraints grows.

We put forth Constrained Primal-Dual PIFM (CPD-PIFM), which adapts Lagrange multipliers during reconstruction.
At each step, the sampler predicts its endpoint, evaluates constraint violations, and uses them to update the multipliers for subsequent steps.
Guidance can use differentiable surrogates or zeroth-order estimates~\cite{tenorio2025graph}, depending on the constraint, but we focus on the former throughout.
The trained PIFM model is reused, and observed entries are re-imposed after every step.
The multipliers therefore control structural properties of the reconstruction while observation consistency holds by construction.

Our contributions are as follows.
(i) We develop CPD-PIFM, a primal-dual graph reconstruction method whose multipliers evolve within each sampling trajectory in response to predicted terminal violations.
(ii) We establish permutation equivariance and an expected terminal slack bound with an $O(1/\sqrt{K})$ term and two approximation terms, where $K$ is the number of integration steps.
(iii) Through experiments on real data, we demonstrate substantial improvement in the feasibility achieved by CPD-PIFM while remaining competitive in terms of accuracy and realism.

\section{Problem Formulation}
\label{sec:formulation}

We represent an undirected graph with $N$ nodes by its binary, symmetric, zero-diagonal adjacency matrix $\bbA \in \{0,1\}^{N \times N}$.
The observation is $\bbA^{\ccalO} = \xi \odot \bbA$, where $\xi$ is a symmetric binary mask that equals one on observed node pairs and $\odot$ denotes the Hadamard product.

\noindent
\textbf{Constrained reconstruction.}
Let $\ccalA(\bbA^{\ccalO}, \xi)$ be the set of symmetric, zero-diagonal matrices $\bbX \in [0,1]^{N \times N}$ satisfying $\xi\odot\bbX=\bbA^{\ccalO}$.
With $m$ constraints indexed by $[m]=\{1,\ldots,m\}$, we study
\begin{align}
\label{eq:crec}
\min_{p} \;& \EE_p\big[\Delta(\bbA, \hbA)\big] \\
\text{subject to} \;\;
& \mathrm{supp}\, p \subseteq \ccalA(\bbA^{\ccalO}, \xi), \nonumber\\
& \EE_p\big[c_l(\hbA) \mid \bbA^{\ccalO}, \xi\big] \le \epsilon_l, \;\; l \in [m], \nonumber
\end{align}
where $p=p(\hbA\mid\bbA^{\ccalO},\xi)$ is a conditional reconstruction distribution and $\Delta$ measures reconstruction error.
We focus on the mean squared error (MSE) loss throughout.
Each statistic $c_l$ has a budget $\epsilon_l$ and may be non-differentiable.
Examples include maximum degree, edge density, and triangle count.
For instance, for a degree cap, $c_l(\hbA)$ is the maximum degree and $\epsilon_l$ is the allowed degree.
Statistics defined on binary graphs act on the binarization of their argument at threshold $0.5$.


\section{CPD-PIFM}
\label{sec:method}

CPD-PIFM augments PIFM's sampling dynamics with multipliers that respond to predicted constraint violations.

\noindent
\textbf{PIFM background.}
PIFM~\cite{chen2026pifm} is a flow matching model with a data-dependent source distribution~\cite{interpolant}.
It first uses a local estimator $f_{\text{prior}}(\bbA^{\ccalO})$ to approximate the posterior mean $\EE[\bbA\mid\bbA^{\ccalO},\xi]$ on the unobserved entries.
This estimator can be a pretrained graph neural network (GNN)~\cite{GraphSAGE10.5555/3294771.3294869, wangneural} or a graphon estimator~\cite{sigl}.
A rectified flow then transports a Gaussian source centered on this estimate toward the distribution of true graphs $\bbA$.
More precisely, PIFM starts from the source sample
\begin{equation}
\label{eq:source}
\bbA_0 = \bbA^{\ccalO} + (\bbone - \xi) \odot \big(f_{\text{prior}}(\bbA^{\ccalO}) + \bbepsilon_s\big),
\end{equation}
where $\bbone$ is the $N\times N$ matrix of ones and $\bbepsilon_s$ is isotropic Gaussian noise on the symmetric, zero-diagonal coordinates.
The target graph is $\bbA_1=\bbA$.
The velocity field $v_\theta$ is trained to predict $\bbA_1-\bbA_0$ on the linear paths $\bbA_t=(1-t)\bbA_0+t\bbA_1$ by minimizing the conditional flow matching loss~\cite{lipmanflow}.
Once $v_\theta$ is trained, sampling produces a trajectory $\hbA_t$ by integrating this velocity with $K$ Euler steps, clipping to $[0,1]$ and re-imposing observed entries after each step.
The resulting sampler targets a distribution $p$ that satisfies the support constraint in~\eqref{eq:crec} but ignores the statistic budgets.
CPD-PIFM reuses its prior and trained flow.

\noindent
\textbf{Endpoint prediction.}
To guide a trajectory toward the budget constraints, we predict the terminal graph from the current state.
Extrapolating the current velocity to $t=1$ yields the conditional mean of the terminal graph given the current state~\cite{martin2025pnpflow}, which we clip to $[0,1]$:
\begin{equation}
\label{eq:extrap}
\bar{\bbA}_{1|t}=\clip\big(\hbA_t+(1-t)v_\theta(\hbA_t,t),0,1\big).
\end{equation}
We then re-impose the observed entries on this prediction:
$\hbA_{1|t}=\bbA^{\ccalO}+(\bbone-\xi)\odot\bar{\bbA}_{1|t}$.
This ensures that constraints are evaluated on a graph that preserves the known entries.
In contrast to PIFM, which integrates $v_\theta$ directly, CPD-PIFM additionally uses~\eqref{eq:extrap} to evaluate terminal constraints during sampling.

\noindent
\textbf{Constraint evaluation and guidance.}
For each constraint $l\in[m]$ in~\eqref{eq:crec}, define the normalized slack $u_l(\bbX)=(c_l(\bbX)-\epsilon_l)/S_l$, where $S_l>0$ is fixed from a known range or training statistics.
A positive slack indicates a violation.
Let $\bbu(\bbX)\in\reals^m$ collect these slacks and let $\hat{\bbu}_k=\bbu(\hbA_{1|t})$ at $t=t_k=k/K$.
Each constraint $l$ has a nonnegative multiplier $\eta_{l,k}$, stacked in $\bbeta_k\in\reals^m_{+}$ with $\bbeta_0=\bbzero$.
As in Lagrangian methods, these multipliers weight each constraint's guidance.
For discrete statistics, binarization at $0.5$ makes the slack gradient zero almost everywhere.
We therefore assume differentiable surrogates $\tilde{u}_l$ for all slacks and evaluate their gradients at $\hbA_{1|t}$;
for example, row sums give degrees and log-sum-exp approximates their maximum.
The guidance direction is
\begin{equation}
\label{eq:guidance}
\bbg_k = -\sum_{l=1}^{m}\eta_{l,k}\,\frac{\Pi_\xi\big(\nabla_{\hbA_{1|t}}\tilde{u}_l(\hbA_{1|t})\big)}{\big\|\Pi_\xi\big(\nabla_{\hbA_{1|t}}\tilde{u}_l(\hbA_{1|t})\big)\big\|_F},
\end{equation}
where $\Pi_\xi$ projects onto the masked, symmetric, zero-diagonal subspace:
$[\Pi_\xi(\bbX)]_{ij}=(1-\xi_{ij})(X_{ij}+X_{ji})/2$ for $i\ne j$, with zero diagonal.
By convention, a term in~\eqref{eq:guidance} is zero when its projected gradient vanishes.
Each projected gradient is normalized before weighting by $\eta_{l,k}$, so gradient magnitudes do not set the relative guidance strengths.

\noindent
\textbf{Coupled sampling updates.}
Exact slacks $\hat{\bbu}_k$ update the multipliers; the guided velocity updates the state.
With dual step size $\rho_k$ and guidance strength $\lambda_k$, the updates at step $k$ are
\begin{align}
\bbeta_{k+1} &= \pos{\bbeta_k+\rho_k\hat{\bbu}_k}, \label{eq:dual-update}\\
\hbA_{t+1/K} &= \clip\big(\hbA_t+\tfrac{1}{K}\big[v_\theta(\hbA_t,t)+\lambda_k\bbg_k\big],0,1\big).
\label{eq:controlled-step}
\end{align}
Here $\pos{\cdot}$ denotes projection onto the nonnegative orthant.
The new multipliers $\bbeta_{k+1}$ affect guidance at step $k+1$.
A warm-up time $\tau_0$ can defer~\eqref{eq:dual-update} until $t\ge \tau_0$ when early endpoint predictions are unreliable; before then, $\bbeta_{k+1}=\bbeta_k$.
After~\eqref{eq:controlled-step}, we re-impose the observed entries as in PIFM.
Clipping and restoring observations ensure $\hbA\in\ccalA(\bbA^{\ccalO},\xi)$ after each step.

Taken together, the pair $(\hbA_t,\bbeta_k)$ evolves as a coupled state, as in state augmentation for safe control and constrained reinforcement learning~\cite{calvofullana2024state, paternain2019constrained, chamon2020probably} and primal-dual inference for constrained diffusion~\cite{hadou2026constrained}.
At $K=1$, guidance uses $\bbeta_0=\bbzero$ and the method reduces to PIFM, so adaptation requires $K>1$.


\noindent
\textbf{Guidance schedule.}
We use $\lambda_k=\bar\lambda/(1-t_k)$ with $\bar\lambda>0$; this increases guidance as the endpoint estimate $\hbA_{1|t}$, on which the constraints are evaluated, becomes more reliable.
No separate fixed multiplier is selected for each constraint; only $\bar\lambda$ is tuned during validation.

\begin{algorithm}[t]
\small
\caption{CPD-PIFM}
\label{alg:cpd}
\begin{algorithmic}[1]
\STATE \textbf{Training:} identical to PIFM; no retraining for any constraints.
\STATE \textbf{Input:} $\bbA^{\ccalO}$, $\xi$; constraints and surrogates $\{(c_l, \epsilon_l, \tilde{u}_l)\}_{l=1}^{m}$; steps $K$; schedules $\rho_k$, $\lambda_k$; warm-up $\tau_0$.
\STATE Initialize $\hbA \leftarrow \bbA^{\ccalO} + (\bbone - \xi) \odot \big(f_{\text{prior}}(\bbA^{\ccalO}) + \bbepsilon_s\big)$; $\bbeta \leftarrow \bbzero$.
\FOR{$k \leftarrow 0, \ldots, K - 1$}
    \STATE $t \leftarrow k / K$
    \STATE $\hbA_{1|t} \leftarrow \bbA^{\ccalO} + (\bbone - \xi) \odot \bar{\bbA}_{1|t}$ with $\bar{\bbA}_{1|t}$ from~\eqref{eq:extrap}
    \STATE $\bbg_k \leftarrow$ guidance from~\eqref{eq:guidance} with current $\bbeta$
    \IF{$t \ge \tau_0$}
        \STATE $\hat{\bbu}_k \leftarrow \bbu\big(\hbA_{1|t}\big)$ \hfill \textit{// normalized; binarized for discrete $c_l$}
        \STATE $\bbeta \leftarrow \pos{\bbeta + \rho_k\, \hat{\bbu}_k}$ \hfill \textit{// takes effect at step $k+1$}
    \ENDIF
    \STATE $\hbA \leftarrow \clip\big(\hbA + \tfrac{1}{K}\big[v_{\theta}(\hbA, t) + \lambda_k \bbg_k\big], 0, 1\big)$
    \STATE $\hbA \leftarrow \bbA^{\ccalO} + (\bbone - \xi) \odot \hbA$ \hfill \textit{// re-impose observed entries}
\ENDFOR
\STATE Return $\hbA$
\end{algorithmic}
\end{algorithm}

\section{Theoretical Guarantees}
\label{sec:theory}

\begin{figure*}[t]
\centering
\includegraphics[width=\textwidth]{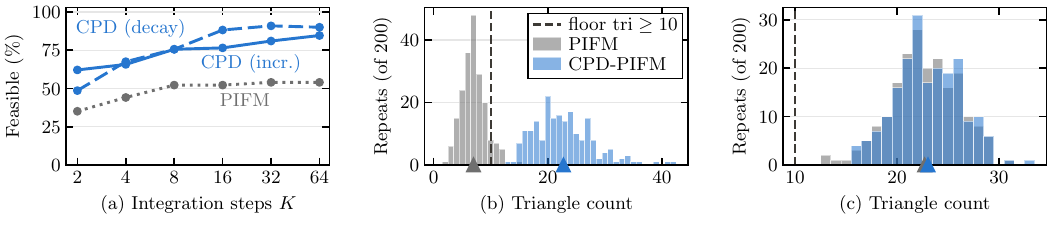}

\caption{(a) PROTEINS with competing edge-floor and degree-cap constraints. Curves show feasibility versus $K$ for PIFM and for CPD-PIFM with increasing and decaying schedules for $\lambda_k$. (b)--(c) ENZYMES with $\mathrm{tri}\ge 10$. Each panel shows triangle counts on full reconstructions from 200 paired samples of one partially observed graph; only the Gaussian source noise changes. In (b),  $P(\mathrm{feasible})$ changes from $0.105$ to $1.00$, and the mean (triangle markers) changes from $7.0$ to $22.7$. In (c), feasibility stays at $1.00$, and the mean changes from $22.7$ to $23.0$.}
\label{fig:k-coupling}
\end{figure*}

Permutation equivariance ensures that relabeling the nodes only relabels the reconstruction.
We establish this property under two assumptions.
\\
\emph{(A1)}~$f_{\text{prior}}$ and $v_\theta$ are permutation-equivariant and map symmetric zero-diagonal matrices to symmetric zero-diagonal matrices.
\\
\emph{(A2)}~Each $c_l$ and each surrogate $\tilde{u}_l$ is permutation-invariant.\\
Assumption (A1) retains PIFM's symmetry assumptions on its prior and velocity field.
Assumption (A2) holds for the graph statistics mentioned in Section~\ref{sec:formulation} and their invariant surrogates.

\begin{theorem}[Equivariance is inherited]
\label{thm:equiv}
Let (A1)--(A2) hold and let $\mathrm{CPD}(\bbA^{\ccalO}, \xi)$ denote the output of Algorithm~\ref{alg:cpd}.
Then, for every permutation matrix $\bbP$,
\begin{equation}
\mathrm{CPD}\big(\bbP^\top \bbA^{\ccalO} \bbP,\, \bbP^\top \xi\, \bbP\big) \;\stackrel{d}{=}\; \bbP^\top\, \mathrm{CPD}\big(\bbA^{\ccalO}, \xi\big)\, \bbP.
\end{equation}
Equality holds pointwise when the source noise is zero.
\end{theorem}

\emph{Proof.} Fix a permutation matrix $\bbP$ and write $\Phi(\bbX) = \bbP^\top \bbX \bbP$; $\Phi$ is orthogonal for the Frobenius inner product, maps the masked symmetric zero-diagonal subspace of $\xi$ to that of $\xi' = \Phi(\xi)$, and commutes with entrywise maps ($\clip$, binarization) and Hadamard products.
Couple the runs on $(\bbA^{\ccalO}, \xi)$ and $(\Phi(\bbA^{\ccalO}), \xi')$ by relabeling the source noise as $\bbepsilon_s' = \Phi(\bbepsilon_s)$.
We claim, by induction along Algorithm~\ref{alg:cpd}, that $\hbA'_t = \Phi(\hbA_t)$ and $\bbeta'_k = \bbeta_k$; the initialization satisfies both by equivariance of $f_{\text{prior}}$.
For the inductive step, the extrapolation relabels by equivariance of $v_\theta$.
The slacks coincide, $\hat{\bbu}'_k = \hat{\bbu}_k$, since binarization commutes with $\Phi$ and each $u_l$ is invariant by (A2), so the dual updates coincide.
The guidance relabels, $\bbg'_k = \Phi(\bbg_k)$, since invariance and orthogonality give $\nabla \tilde{u}_l \circ \Phi = \Phi \circ \nabla \tilde{u}_l$ and $\Pi_{\xi'} \circ \Phi = \Phi \circ \Pi_\xi$, while $\|\Phi(\bbX)\|_F=\|\bbX\|_F$ preserves the normalization in~\eqref{eq:guidance}.
Finally, the update~\eqref{eq:controlled-step} is a linear combination of relabeling terms followed by entrywise operations.
Hence $\mathrm{CPD}(\Phi(\bbA^{\ccalO}), \xi') = \Phi(\mathrm{CPD}(\bbA^{\ccalO}, \xi))$ pathwise under the coupling, and in distribution since the coupling preserves the law.
With zero source noise both runs are deterministic and the identity is pointwise.
\hfill$\blacksquare$

Theorem~\ref{thm:equiv} shows that invariant constraints preserve equivariance.
The reconstruction therefore respects relabeling of the nodes.

We next bound the expected terminal slack using three assumptions.\\
\emph{(A3)} The normalized slacks satisfy $|u_l(\bbX)| \le B$ on $[0,1]^{N \times N}$; let $\bar{B} = \sqrt{m}\,B$.\\
\emph{(A4)} Guidance effectiveness: there exists $\varepsilon_g \ge 0$ with
\begin{equation}
\label{eq:A4}
\EE\bigg[\frac{1}{K}\sum_{k=0}^{K-1} \langle \bbeta_k, \hat{\bbu}_k \rangle\bigg] \le \varepsilon_g.
\end{equation}
Here, $\hat{\bbu}_k=\bbu(\hbA_{1|t_k})$ is evaluated along the guided trajectory produced by~\eqref{eq:controlled-step}; (A4) bounds the expected time average of its multiplier-weighted sum by $\varepsilon_g$.\\
\emph{(A5)} Endpoint consistency: there exists $\bar{e} \ge 0$ such that, for every $l \in [m]$,
\begin{equation}
\label{eq:A5}
\EE\bigg[\frac{1}{K}\sum_{k=0}^{K-1} \big| u_l(\hbA) - \hat{u}_{l,k} \big|\bigg] \le \bar{e},
\end{equation}
where $\hbA$ is the realized terminal reconstruction.\\
Assumption (A3) holds for continuous statistics on the compact cube and for bounded discrete statistics.
The errors $\varepsilon_g$ and $\bar{e}$ in (A4)--(A5) quantify guidance effectiveness and the gap between predicted and realized terminal slacks.

\begin{theorem}[Terminal violation bound]
\label{thm:violation}
Let (A3)--(A5) hold, initialize $\bbeta_0 = \bbzero$, set $\rho_k = \rho =  1/(\bar{B}\sqrt{K})$, and take $\tau_0 = 0$. 
Then, for every constraint $l \in [m]$, the terminal normalized slack of Algorithm~\ref{alg:cpd} satisfies
\begin{equation}
\label{eq:violation-bound}
\EE\big[u_l\big(\hbA\big)\big] \;\le\; \frac{\bar{B}}{\sqrt{K}} \;+\; \varepsilon_g \;+\; \bar{e}.
\end{equation}
\end{theorem}

\emph{Proof.} Set $\rho = 1/(\bar{B}\sqrt{K})$. 
The dual recursion $\bbeta_{k+1} = \pos{\bbeta_k + \rho \hat{\bbu}_k}$ is projected online gradient ascent against the slack sequence~\cite{zinkevich2003online, mahdavi2012trading}.
Fix a comparator $\bbeta^* \ge \bbzero$.
Since $\bbeta^*$ lies in the nonnegative orthant and the projection $\pos{\cdot}$ onto it is nonexpansive, $\|\bbeta_{k+1} - \bbeta^*\|^2 \le \|\bbeta_k + \rho \hat{\bbu}_k - \bbeta^*\|^2 = \|\bbeta_k - \bbeta^*\|^2 + 2\rho \langle \hat{\bbu}_k, \bbeta_k - \bbeta^* \rangle + \rho^2 \|\hat{\bbu}_k\|^2$, with $\|\hat{\bbu}_k\|^2 \le \bar{B}^2$ by (A3).
Summing over $k = 0, \ldots, K-1$, telescoping, dropping $\|\bbeta_K - \bbeta^*\|^2 \ge 0$, and using $\bbeta_0 = \bbzero$ gives the pathwise regret bound $\sum_{k} \langle \hat{\bbu}_k, \bbeta^* - \bbeta_k \rangle \le \|\bbeta^*\|^2/(2\rho) + \rho K \bar{B}^2/2$.
Take $\bbeta^* = \bbe_l$, the $l$-th standard basis vector, divide by $K$, and use $1/(2\rho K) = \rho \bar{B}^2/2 = \bar{B}/(2\sqrt{K})$ to obtain, pathwise, $\frac{1}{K}\sum_{k} \hat{u}_{l,k} \le \bar{B}/\sqrt{K} + \frac{1}{K}\sum_{k} \langle \bbeta_k, \hat{\bbu}_k \rangle$.
Since, identically, $u_l(\hbA) = \frac{1}{K}\sum_{k} \big[\hat{u}_{l,k} + \big(u_l(\hbA) - \hat{u}_{l,k}\big)\big] \le \frac{1}{K}\sum_{k} \hat{u}_{l,k} + \frac{1}{K}\sum_{k} \big|u_l(\hbA) - \hat{u}_{l,k}\big|$, combining the two previous equations and taking expectations, \eqref{eq:A4} and~\eqref{eq:A5} bound the two random averages by $\varepsilon_g$ and $\bar{e}$, which is~\eqref{eq:violation-bound}.
\hfill$\blacksquare$

The bound separates the cost of adapting the multipliers, $\bar B/\sqrt K$, from the guidance error $\varepsilon_g$ and endpoint error $\bar e$.
Only the first term is required to decay with $K$; thus, the theorem does not guarantee hard feasibility.
Finally, although the theorem is shown for $\tau_0=0$, an analogous result holds for $\tau_0 > 0$ where we set $\rho_k = 1/(\bar{B}\sqrt{M})$ on the $M = K - \lceil \tau_0 K \rceil \ge 1$ active dual steps; the bound then holds with $K$ replaced by $M$ and the averages in (A4)--(A5) taken over those steps.

\section{Experiments}
\label{sec:experiments}

\textbf{Setup.} To isolate the effect of sampling guidance, we reuse PIFM's datasets, masks, splits, GraphSAGE~\cite{GraphSAGE10.5555/3294771.3294869} priors, and trained velocity checkpoints.
We evaluate link prediction on ENZYMES, PROTEINS, and IMDB~\cite{TUDataset}.
In each graph, $50\%$ of node pairs are hidden uniformly at random.
Unless stated otherwise, all methods use Euler sampling with $K=32$ and $\tau_0=0$.
For each dataset and constraint setting, we select $\bar\lambda$ on the seed-0 validation split as the smallest value attaining the highest feasibility, and reuse it across seeds.
PIFM corresponds to $\bar\lambda=0$.
As a baseline without dual adaptation, fixed guidance replaces every multiplier $\eta_{l,k}$ by one constant $\eta$, keeping $\bar\lambda=1$, the schedule, and the normalized directions in~\eqref{eq:guidance}.
We report three metrics.
Feasibility is the fraction of test graphs whose reconstruction satisfies every constraint.
Accuracy is the area under the receiver operating characteristic curve (AUC) on unobserved pairs, averaged over test graphs whose ground truth satisfies all the constraints.
Realism is the squared maximum mean discrepancy (MMD) with a Gaussian earth-mover kernel between the degree histograms of all reconstructions and of the ground-truth graphs in that subset.
Unless noted, results are means and standard deviations (s.d.) over five seeds, each with its own split, masks, prior, checkpoint, and budgets.

\noindent
\textbf{Constraints.} We set all budgets $\epsilon_l$ from whole-graph training-split statistics and evaluate each $c_l$ on the full reconstruction.
Guidance itself acts only on unobserved entries through $\Pi_\xi$ in~\eqref{eq:guidance}.
The single constraints are a degree cap at $q_{0.9}$, where $q_\alpha$ is the $\alpha$-quantile of the training statistic, a triangle floor at $q_{0.1}$ (normalized by $\binom{N}{3}$ on PROTEINS), and an edge-density band between $q_{0.1}$ and $q_{0.9}$.
The mixture combines all three.
The competing pair combines an edge floor at $q_{0.25}$ with a degree cap at $q_{0.5}$ of the maximum degree on ENZYMES and $q_{0.8}$ of the maximum degree normalized by $N-1$ on PROTEINS.
The floor encourages more edges, while the cap limits their concentration.

\begin{table}[t]
\centering
\caption{Effect of CPD guidance by constraint at $K=32$, averaged over five test seeds. S.d.\ is shown for feasibility only; the s.d.\ of the paired AUC and MMD changes is at most $0.010$ and $0.07$.}
\label{tab:exp1}
\setlength{\tabcolsep}{2.3pt}
\renewcommand{\arraystretch}{1.05}
\begin{tabular*}{\columnwidth}{@{\extracolsep{\fill}}l@{\hspace{5pt}}rr@{\hspace{8pt}}rr@{\hspace{8pt}}rr@{}}
\toprule
 & \multicolumn{2}{c}{Feas.\ \% $\uparrow$} & \multicolumn{2}{c}{AUC $\uparrow$} & \multicolumn{2}{c}{MMD $\downarrow$} \\
\cmidrule(lr){2-3}\cmidrule(lr){4-5}\cmidrule(lr){6-7}
Constraint & PIFM & CPD & PIFM & CPD & PIFM & CPD \\
\midrule
\multicolumn{7}{@{}l}{\textit{ENZYMES}} \\
\;deg.\ cap & $83.2\pm 7.4$ & $\mathbf{99.4\pm 1.4}$ & .550 & .554 & .205 & .204 \\
\;tri.\ floor & $74.2\pm 4.0$ & $\mathbf{99.4\pm 1.4}$ & .552 & .579 & .222 & .318 \\
\;mixture & $47.1\pm 5.9$ & $\mathbf{71.0\pm 6.0}$ & .560 & .576 & .243 & .199 \\
\;competing & $43.9\pm 10.4$ & $\mathbf{69.7\pm 8.7}$ & .551 & .561 & .225 & .172 \\
\addlinespace[2pt]
\multicolumn{7}{@{}l}{\textit{IMDB}} \\
\;tri.\ floor & $84.0\pm 4.2$ & $\mathbf{99.6\pm 0.9}$ & .879 & .882 & .030 & .039 \\
\addlinespace[2pt]
\multicolumn{7}{@{}l}{\textit{PROTEINS}} \\
\;deg.\ cap & $88.6\pm 5.6$ & $\mathbf{99.6\pm 0.8}$ & .586 & .584 & .241 & .242 \\
\;tri.\ floor & $86.1\pm 5.6$ & $\mathbf{99.3\pm 1.6}$ & .590 & .593 & .272 & .253 \\
\;edge band & $80.7\pm 8.6$ & $\mathbf{96.1\pm 4.1}$ & .573 & .574 & .259 & .287 \\
\;competing & $53.6\pm 7.9$ & $\mathbf{72.5\pm 9.8}$ & .576 & .563 & .279 & .489 \\
\bottomrule
\end{tabular*}
\end{table}

\begin{table}[t]
\centering
\caption{Fixed guidance and CPD-PIFM at $K=32$, single seed. Both fixed multipliers $\eta$ are shared across all settings; PROTEINS edge-band gradients cancel under shared $\eta$, so fixed guidance equals PIFM. Feasibility is reported as a percentage; the highest in each setting is in bold.}
\label{tab:fixed}
\setlength{\tabcolsep}{1.4pt}
\renewcommand{\arraystretch}{1.05}
\begin{tabular*}{\columnwidth}{@{\extracolsep{\fill}}l@{\hspace{4pt}}rrr@{\hspace{6pt}}rrr@{\hspace{6pt}}rrr@{}}
\toprule
 & \multicolumn{3}{c}{Fixed $\eta=0.1$} & \multicolumn{3}{c}{Fixed $\eta=1.6$} & \multicolumn{3}{c}{CPD-PIFM} \\
\cmidrule(lr){2-4}\cmidrule(lr){5-7}\cmidrule(lr){8-10}
Constraint & Feas & AUC & MMD & Feas & AUC & MMD & Feas & AUC & MMD \\
\midrule
\multicolumn{10}{@{}l}{\textit{ENZYMES}} \\
deg.\ cap & 93.5 & .569 & .210 & \textbf{100} & .572 & .597 & \textbf{100} & .567 & .195 \\
tri.\ floor & 74.2 & .578 & .195 & \textbf{100} & .618 & 1.007 & \textbf{100} & .590 & .385 \\
mixture & 58.1 & .603 & .204 & \textbf{80.6} & .681 & .199 & 74.2 & .599 & .168 \\
competing & 58.1 & .567 & .213 & 74.2 & .610 & .081 & \textbf{80.6} & .568 & .184 \\
\addlinespace[2pt]
\multicolumn{10}{@{}l}{\textit{IMDB}} \\
tri.\ floor & 90.0 & .877 & .022 & \textbf{100} & .840 & .186 & \textbf{100} & .878 & .032 \\
\addlinespace[2pt]
\multicolumn{10}{@{}l}{\textit{PROTEINS}} \\
deg.\ cap & 87.5 & .609 & .265 & 98.2 & .578 & .706 & \textbf{100} & .598 & .263 \\
tri.\ floor & 91.1 & .638 & .264 & 98.2 & .647 & .552 & \textbf{100} & .605 & .262 \\
edge band & 80.4 & .592 & .242 & 80.4 & .592 & .242 & \textbf{96.4} & .595 & .298 \\
competing & 58.9 & .612 & .257 & 57.1 & .650 & .097 & \textbf{80.4} & .582 & .487 \\
\bottomrule
\end{tabular*}
\end{table}

\noindent
\textbf{Effect of integration steps.} Fig.~\ref{fig:k-coupling}(a) shows that CPD-PIFM feasibility improves as $K$ increases (in line with Theorem~\ref{thm:violation}) on the PROTEINS dataset with competing constraints.
The PIFM rate levels off near $54\%$.
With the increasing schedule for $\lambda_k$, feasibility rises from $62\%$ at $K=2$ to $85\%$ at $K=64$.
As an ablation, the decaying schedule $\lambda_k=\bar\lambda(1-t_k)$, which weights early guidance more, reaches higher feasibility at intermediate $K$ but is not monotone in $K$, whereas the increasing schedule improves steadily.
We keep the increasing schedule as the default because it applies guidance when the endpoint estimate is most reliable.
Across the seven single and mixed settings and five test seeds, the average feasibility gain over PIFM is $17.2$ percentage points at $K=32$.

\noindent
\textbf{Effect of guidance.} Figs.~\ref{fig:k-coupling}(b) and (c) illustrate how guidance affects two individual ENZYMES graphs.
For the first graph, $179$ of $200$ PIFM samples miss the triangle floor.
Guidance moves all 200 samples above the floor and raises the mean from $7.0$ to $22.7$ triangles.
The second graph is already feasible under PIFM.
Thus, guidance only minimally changes its mean from $22.7$ to $23.0$ triangles, while feasibility remains at $100\%$. 

\noindent
\textbf{Results across constraints.} Table~\ref{tab:exp1} shows higher feasibility in all nine settings, with gains from $11.1$ to $25.8$ percentage points.
Every single-constraint setting reaches at least $96\%$ feasibility, while the mixture and the competing pairs reach $70$--$73\%$.
Mean AUC changes range from $-0.013$ to $+0.027$; six of nine lie within $\pm0.01$.
Degree MMD rises in four settings, is nearly unchanged on the two degree caps, and falls in three.
In particular, CPD-PIFM improves all three metrics on the ENZYMES mixture and competing constraints.

\noindent
\textbf{Comparison with fixed guidance.} The two shared fixed weights in Table~\ref{tab:fixed} illustrate the tradeoff between feasibility and reconstruction quality across settings.
On the ENZYMES single constraints and the mixture, $\eta=1.6$ matches or exceeds CPD-PIFM in feasibility and AUC, although CPD-PIFM has lower MMD.
The same multiplier transfers poorly to IMDB and the PROTEINS degree cap, where CPD-PIFM matches or improves feasibility and gives better AUC and MMD.
Reducing $\eta$ to $0.1$ preserves degree MMD better in these two settings but lowers feasibility to $90.0\%$ on IMDB and $87.5\%$ on the PROTEINS degree cap, compared with $100\%$ for CPD-PIFM.
At $\eta=1.6$, CPD-PIFM matches or exceeds fixed-guidance feasibility in eight of nine settings and has lower MMD in six, while fixed guidance has higher AUC in six.
We also evaluated $\eta=25.6$ (not shown), which reduced AUC in seven settings and increased MMD in eight relative to the better smaller weight for each metric, while improving feasibility only in two.
Thus, fixed guidance needs setting-specific tuning, whereas CPD-PIFM remains competitive without choosing a separate fixed multiplier for each constraint.

\section{Conclusion}
\label{sec:conclusion}

We developed CPD-PIFM, which guides graph reconstruction with multipliers updated from predicted endpoint violations during sampling.
It reuses the trained flow, preserves observed entries, inherits permutation equivariance under invariant constraints, and needs one validated guidance scale.
Experiments show substantially higher feasibility with competitive accuracy and realism.

\section{Compliance with Ethical Standards}
This is a numerical study on publicly available graph datasets, for which no ethical approval was required.

\section{Acknowledgments}
This work was supported by NSF under Award CCF 2340481. The authors have no relevant financial or non-financial interests to disclose.
Claude (Anthropic) was used to assist with drafting and editing the text of Section~\ref{sec:experiments}, with experiment and plotting code, and with assembling Fig.~\ref{fig:k-coupling} and Tables~\ref{tab:exp1}--\ref{tab:fixed}.
All content was reviewed, verified, and revised by the authors, who take full responsibility for it.

\bibliographystyle{IEEEbib}
\bibliography{refs}

\begin{thebibliography}{10}

\bibitem{segarra2017network}
Santiago Segarra, Antonio~G. Marques, Gonzalo Mateos, and Alejandro Ribeiro,
\newblock ``Network topology inference from spectral templates,''
\newblock {\em IEEE Trans. Signal and Info. Process. over Networks}, vol. 3,
  no. 3, pp. 467--483, 2017.

\bibitem{dong2016learning}
Xiaowen Dong, Dorina Thanou, Pascal Frossard, and Pierre Vandergheynst,
\newblock ``Learning {Laplacian} matrix in smooth graph signal
  representations,''
\newblock {\em IEEE Trans. Signal Process.}, vol. 64, no. 23, pp. 6160--6173,
  2016.

\bibitem{mateos2019connecting}
Gonzalo Mateos, Santiago Segarra, Antonio~G. Marques, and Alejandro Ribeiro,
\newblock ``Connecting the dots: Identifying network structure via graph signal
  processing,''
\newblock {\em IEEE Signal Process. Mag.}, vol. 36, no. 3, pp. 16--43, 2019.

\bibitem{roddenberry2021network}
T~Mitchell Roddenberry, Madeline Navarro, and Santiago Segarra,
\newblock ``Network topology inference with graphon spectral penalties,''
\newblock in {\em IEEE Intl. Conf. Acoust., Speech and Signal Process.
  (ICASSP)}. IEEE, 2021, pp. 5390--5394.

\bibitem{navarro2022joint}
Madeline Navarro and Santiago Segarra,
\newblock ``Joint network topology inference via a shared graphon model,''
\newblock {\em IEEE Trans. Signal Process.}, vol. 70, pp. 5549--5563, 2022.

\bibitem{tenorio2025graph}
Victor~M Tenorio, Nicolas Zilberstein, Santiago Segarra, and Antonio~G Marques,
\newblock ``Graph guided diffusion: Unified guidance for conditional graph
  generation,''
\newblock {\em arXiv preprint arXiv:2505.19685}, 2025.

\bibitem{sharma2024diffuse}
Kartik Sharma, Srijan Kumar, and Rakshit Trivedi,
\newblock ``Diffuse, sample, project: plug-and-play controllable graph
  generation,''
\newblock in {\em Intl. Conf. on Machine Learning (ICML)}, 2024.

\bibitem{chung2023solving}
Hyungjin Chung, Dohoon Ryu, Michael~T. McCann, Marc~L. Klasky, and Jong~Chul
  Ye,
\newblock ``Solving 3{D} inverse problems using pre-trained 2{D} diffusion
  models,''
\newblock in {\em Proceedings of the IEEE/CVF Int. Conf. Comput. Vis. Pattern
  Recogn. (CVPR)}, 2023, pp. 22542--22551.

\bibitem{kim2025dual}
Minseo Kim, Axel Levy, and Gordon Wetzstein,
\newblock ``Dual ascent diffusion for inverse problems,''
\newblock in {\em Proceedings of the IEEE/CVF Int. Conf. Comput. Vis. Pattern
  Recogn. (CVPR)}, 2026.

\bibitem{tomasi2026primaldual}
Federico Tomasi, Dmitrii Moor, Alice Wang, and Mounia Lalmas,
\newblock ``Primal-dual guided decoding for constrained discrete diffusion,''
\newblock {\em arXiv preprint arXiv:2605.09749}, 2026.

\bibitem{hadou2026constrained}
Samar Hadou, Yigit~Berkay Uslu, and Alejandro Ribeiro,
\newblock ``Constrained diffusion models with primal-dual inference,''
\newblock {\em arXiv preprint arXiv:2606.17192}, 2026.

\bibitem{chen2026pifm}
Harvey Chen, Nicolas Zilberstein, and Santiago Segarra,
\newblock ``Prior-informed flow matching for graph reconstruction,''
\newblock {\em arXiv preprint arXiv:2601.22107}, 2026.

\bibitem{liuflow}
Xingchao Liu, Chengyue Gong, and Qiang Liu,
\newblock ``Flow straight and fast: Learning to generate and transfer data with
  rectified flow,''
\newblock in {\em Intl. Conf. Learn. Repr. (ICLR)}, 2023.

\bibitem{lipmanflow}
Yaron Lipman, Ricky~TQ Chen, Heli Ben-Hamu, Maximilian Nickel, and Matthew Le,
\newblock ``Flow matching for generative modeling,''
\newblock in {\em Intl. Conf. Learn. Repr. (ICLR)}, 2023.

\bibitem{blau2018perception}
Yochai Blau and Tomer Michaeli,
\newblock ``The perception-distortion tradeoff,''
\newblock in {\em Proceedings of the IEEE/CVF Int. Conf. Comput. Vis. Pattern
  Recogn. (CVPR)}, 2018, pp. 6228--6237.

\bibitem{freirich2021distortion}
Dror Freirich, Tomer Michaeli, and Ron Meir,
\newblock ``A theory of the distortion-perception tradeoff in {Wasserstein}
  space,''
\newblock in {\em Advances in Neural Inf. Process. Syst. (NeurIPS)}, 2021.

\bibitem{interpolant}
Michael~Samuel Albergo, Mark Goldstein, Nicholas~Matthew Boffi, Rajesh
  Ranganath, and Eric Vanden-Eijnden,
\newblock ``Stochastic interpolants with data-dependent couplings,''
\newblock in {\em Intl. Conf. on Machine Learning (ICML)}. PMLR, 2024, pp.
  921--937.

\bibitem{GraphSAGE10.5555/3294771.3294869}
William~L. Hamilton, Rex Ying, and Jure Leskovec,
\newblock ``Inductive representation learning on large graphs,''
\newblock in {\em Advances in Neural Inf. Process. Syst. (NIPS)}, 2017,
  NIPS'17, pp. 1025--1035.

\bibitem{wangneural}
Xiyuan Wang, Haotong Yang, and Muhan Zhang,
\newblock ``Neural common neighbor with completion for link prediction,''
\newblock in {\em Intl. Conf. Learn. Repr. (ICLR)}, 2024.

\bibitem{sigl}
Ali Azizpour, Nicolas Zilberstein, and Santiago Segarra,
\newblock ``Scalable implicit graphon learning,''
\newblock in {\em Int. Conf. on Artif. Intell. and Stat.}, 2025, pp.
  3952--3960.

\bibitem{martin2025pnpflow}
S{\'e}gol{\`e}ne Martin, Anne Gagneux, Paul Hagemann, and Gabriele Steidl,
\newblock ``{PnP-Flow}: Plug-and-play image restoration with flow matching,''
\newblock in {\em Intl. Conf. Learn. Repr. (ICLR)}, 2025.

\bibitem{calvofullana2024state}
Miguel Calvo-Fullana, Santiago Paternain, Luiz F.~O. Chamon, and Alejandro
  Ribeiro,
\newblock ``State augmented constrained reinforcement learning: Overcoming the
  limitations of learning with rewards,''
\newblock {\em IEEE Trans. Auto. Control}, vol. 69, no. 7, pp. 4275--4290,
  2024.

\bibitem{paternain2019constrained}
Santiago Paternain, Luiz F.~O. Chamon, Miguel Calvo-Fullana, and Alejandro
  Ribeiro,
\newblock ``Constrained reinforcement learning has zero duality gap,''
\newblock in {\em Advances in Neural Inf. Process. Syst. (NeurIPS)}, 2019.

\bibitem{chamon2020probably}
Luiz F.~O. Chamon and Alejandro Ribeiro,
\newblock ``Probably approximately correct constrained learning,''
\newblock in {\em Advances in Neural Inf. Process. Syst. (NeurIPS)}, 2020.

\bibitem{zinkevich2003online}
Martin Zinkevich,
\newblock ``Online convex programming and generalized infinitesimal gradient
  ascent,''
\newblock in {\em Intl. Conf. on Machine Learning (ICML)}, 2003, pp. 928--936.

\bibitem{mahdavi2012trading}
Mehrdad Mahdavi, Rong Jin, and Tianbao Yang,
\newblock ``Trading regret for efficiency: Online convex optimization with long
  term constraints,''
\newblock {\em J. Mach. Learn. Res.}, vol. 13, pp. 2503--2528, 2012.

\bibitem{TUDataset}
Christopher Morris, Nils~M. Kriege, Franka Bause, Kristian Kersting, Petra
  Mutzel, and Marion Neumann,
\newblock ``{TUDataset}: A collection of benchmark datasets for learning with
  graphs,''
\newblock in {\em ICML Work. on Graph Rep. Learning and Beyond}, 2020.

\end{thebibliography}

\end{document}